\documentclass[twocolumn]{article}

\usepackage{subfigure}

\usepackage[breaklinks=true]{hyperref}

\usepackage{adjustbox}
\usepackage{natbib}

\usepackage[utf8]{inputenc} 
\usepackage[T1]{fontenc}    
\usepackage{hyperref}       
\usepackage{url}            
\usepackage{booktabs}       
\usepackage{amsfonts}       
\usepackage{nicefrac}       
\usepackage{microtype}      

\usepackage{times}
\usepackage{tikz}
\usepackage{soul}
\usepackage{color}
\usepackage{xcolor}
\usepackage{algorithm}
\usepackage{algorithmic}
\usepackage{enumerate}
\usepackage{comment}
\usepackage{amsthm}
\usepackage{amsmath}
\usepackage{url}

\graphicspath{{/figures_simulations/}}

\newcommand{\one}{{\bf I}}  
\newcommand{\Exp}{\mathbf{E}}

\newcommand{\R}{{\mathbb R}}

\newcommand{\cN}{{\cal N}}

\newcommand{\bx}{{\bf x}}
\newcommand{\bX}{{\bf X}}

\newcommand{\bZ}{{\bf Z}}

\newcommand{\bb}{{\bf b}}
\newcommand{\ba}{{\bf a}}
\newcommand{\bc}{{\bf c}}

\newcommand{\sx}{\Sigma_{\bX\bX}}

\newcommand{\sxy}{\Sigma_{\bX Y}}

\newtheorem{Theorem}{Theorem}
\newtheorem{Lemma}{Lemma}
\newtheorem{Definition}{Definition}

\definecolor{lightgray}{gray}{0.85}
\sethlcolor{Dandelion}

\usetikzlibrary{arrows,shapes,plotmarks,positioning}
\usetikzlibrary{decorations.markings}

\tikzset{>=stealth'} 
\tikzstyle{graphnode} = 
   [circle,draw=black,minimum size=22pt,text centered,text
     width=22pt,inner sep=0pt] 
\tikzstyle{var}   =[graphnode,fill=white]
\tikzstyle{vardashed}   =[graphnode,draw=gray,fill=white]
\tikzstyle{obs}   =[graphnode,fill=black,text=white]
\tikzstyle{obsgrey}   =[graphnode,draw=white,fill=lightgray,text=black]
\tikzstyle{par}    =[graphnode,draw=white,fill=red,text=black] 
 \tikzstyle{crucial} =[graphnode,draw=white,fill=yellow,text=black] 
\tikzstyle{fac}   =[rectangle,draw=black,fill=black!25,minimum size=5pt]
\tikzstyle{facprior} =[rectangle,draw=black,fill=black,text=white,minimum size=5pt]
\tikzstyle{edge}  =[draw=white,double=black,very thick,-]
\tikzstyle{blueedge}  =[draw=white,double=blue,very thick,-]
\tikzstyle{rededge}  =[draw=white,double=red,very thick,-]
\tikzstyle{prior} =[rectangle, draw=black, fill=black, minimum size=
5pt, inner sep=0pt]
\tikzstyle{dirprior} = [circle, draw=black, fill=black, minimum
size=5pt, inner sep=0pt]

\tikzstyle{dot_node}=[draw=black,fill=black,shape=circle]

\date{02 March 2018}

\title{Detecting non-causal artifacts \\in multivariate linear regression models}

\author{Dominik Janzing and Bernhard Sch\"olkopf\\
{\small Max Planck Institute for Intelligent Systemes, T\"ubingen, Germany}}

\hypersetup{draft} 
\begin{document}

\maketitle

\begin{abstract}
We consider linear models where $d$ potential causes $X_1,\dots,X_d$ are correlated with one
target quantity $Y$ and propose a method to infer whether the association is causal or whether 
it is an artifact caused by overfitting or hidden common causes.
We employ the idea that in the former case the vector of regression coefficients has `generic' orientation relative to the covariance matrix
$\Sigma_{XX}$  of $X$. Using an ICA based model for confounding, we show that both confounding and
overfitting yield regression vectors 
that concentrate mainly in the space of
low eigenvalues of $\Sigma_{XX}$. 
\end{abstract}

\section{Introduction}

Inferring causal relations from passive observations data has gained increasing interest in machine learning and statistics. Although  reliable causal conclusions can only be drawn from interventional data, the idea of postulating assumptions that render causal inference from passive observations feasible 
becomes more and more accepted. In addition to the more `traditional' {\it causal Markov condition} and {\it causal faithfulness assumption} \cite{Spirtes1993,Pearl2000}, researchers have also stated assumptions that admit causal inference when no conditional statistical independences hold, e.g., \citet{Kano2003,SunLauderdale,Hoyer,Zhang_UAI,Bloebaum17}.  
Each of these method relies on idealized assumptions that rarely hold in practice; nevertheless they can be useful if the methods possess a degree of robustness regarding violation of model assumptions \cite{Mooij2016}. In a similar vein, the present work considers a causal inference problem that becomes solvable only under an idealized model assumption that is certainly debatable. However, it illustrates 
that high-dimensional observations contain a kind of causal information that has not been employed so far. 

We 
assume that we are given a scalar target variable $Y$ that is potentially influenced by a multi-dimensional predictor variable $\bX=(X_1,\dots,X_d)$. Suppose that i.i.d.\ samples from $P_{\bX,Y}$ show that $\bX$ and $Y$ are significantly correlated, but it is unclear whether this is mainly due to the influence of $\bX$ on $Y$ or due to a common cause of $\bX$ and $Y$ (here we assume that
prior knowledge excludes the case where $Y$ causally influences $\bX$, e.g, due to time order). $Y$ may, for instance, be a quantitiave property of a material (e.g., electrical resistence) and $X_j$ some  features
describing its chemical and physical structure.
In biology, $\bX$ and $Y$ could represent information about genotype and phenotype, respectively. 
Note that conditional independences allow to decide which
of the variables $X_j$ influence $Y$, given 
that the association between $\bX$ and $Y$ is unconfounded. The question of unconfoundedness, which we address here, is therefore prior to the former problem. 

 \hypersetup{draft} 
\begin{figure}
\begin{center}
\resizebox{2cm}{!}{
  \begin{tikzpicture}

    \node[obs] at (4,1) (Xc) {$\bX$} ;
    \node[var] at (5,2) (Z) {$\bZ$} edge[->] (Xc) ;  
    \node[obs] at (6,1) (Yc) {$Y$} edge[<-] (Z) edge[<-] (Xc) ;

  \end{tikzpicture}
}
\hfill
\resizebox{2cm}{!}{
  \begin{tikzpicture}

    \node[obs] at (4,1) (Xc) {$\bX$} ;
    \node[var] at (5,2) (Z) {$\bZ$} edge[->] (Xc) ;  
    \node[obs] at (6,1) (Yc) {$Y$} edge[<-] (Z);

  \end{tikzpicture}
}
\hfill
\resizebox{2cm}{!}{
  \begin{tikzpicture}

    \node[obs] at (4,1) (Xc) {$\bX$} ;
    \node[var] at (5,2) (Z) {$\bZ$} edge[->] (Xc) ;  
    \node[obs] at (6,1) (Yc) {$Y$} edge[<-] (Xc) ;

  \end{tikzpicture}
}
\end{center}
\caption{\label{fig:dags} Generic scenario where the statistical  relation between $\bX$ and $Y$ is due to an unobserved confounder $\bZ$ and due to the influence of $\bX$ on $Y$.
The purely confounded (middle) and the purely causal (right) scenarios are obtained as limiting cases where one of the arrows
is negligible. 
}  
\end{figure}
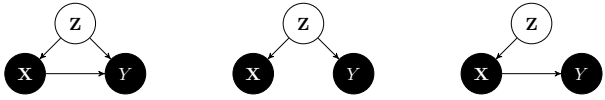

Figure~\ref{fig:dags}, left, visualizes the generic scenario that we consider throughout the paper, where the statistical dependences between $\bX$ and $Y$ are due to 
the influence of $\bX$ on $Y$ and due to the common cause $\bZ$. 
It contains the purely confounded case (middle)
as limiting case where the arrow from $\bX$ to $Y$ is arbitrarily weak. Likewise, the purely causal case is obtained when 
one of the arrows from $\bZ$ gets weak (right).

Our confounder detection is based on 
observing `non-generic' relations between $P_\bX$ and $P_{Y|\bX}$ \cite{multivariateConfound}.
 We thus follow the abstract principle of independent mechanisms \cite{causality_book}, stating that for the purely causal relation $X\to Y$
of two arbitrary variables variables  $X,Y$, the marginal $P_X$ 
 and the conditional $P_{Y|X}$
do not contain information about each other (where `information' needs to be further specified). 
The present paper contains the following novel contributions:

$\bullet$ We allow for multi-dimensional confounders. In contrast, the entire analysis of \citet{multivariateConfound} is restricted to the case of a one-dimensional confounder, and cannot be extended using the methods presented in that work.

$\bullet$ We show that the multivariate setting permits an analysis which is significantly simpler, and also the `dependences' between $P_\bX$ and $P_{Y|\bX}$ become simple. 

$\bullet$ We derive a statistical test for non-confounding based on our model assumptions.

$\bullet$ We show that for our model, overfitting generates the same kind of dependences between
$P_\bX$ and $P_{Y|\bX}$ as confounding. This suggests a subtle link between 
regularization and the correction of confounding. 
One may conjecture, for instance, that models with `independent' $P_\bX$ and $P_{Y|\bX}$
have better chances to generalize to future data points as well as to related 
data sets from other domains (including interventional data), cf.\ also \citet{anticausal}.

It may sound counter-intuitive that the multivariate case can be simpler than the scalar case, but
our derivations are based on a certain notion of 
genericity of the multivariate confounder which does not necessarily hold for the scalar case, although our experiments will also include data with scalar confounding.

\section{Model for confounding with uncorrelated sources}

Our model for the influence of the high-dimensional 
common cause $\bZ$ on both $\bX$ and $Y$ is inspired by
Independent Component Analysis (ICA) \cite{ICA}. Let $\bZ$ consist 
of $\ell\geq d$ independent sources\footnote{In contrast to ICA, however, it is actually enough that the sources are uncorrelated.} $Z_1,\dots,Z_\ell$, each having unit variance. They influence 
$\bX$ via a mixing matrix $M$ and $Y$ via a mixing vector $\bc$,
as shown in Figure~\ref{fig:sourcemodel}. 
Explicitly, the structural equations relating $\bZ,\bX,Y$ thus read:
\begin{eqnarray}
\bX &=& M \bZ \\
Y   &=& \ba^T \bX + \bc^T \bZ, 
\end{eqnarray}
where $M$ is a $d\times \ell$ matrix and $\ba$ are $\bc$ are
vectors in $\R^d$ and $\R^\ell$, respectively. 
\begin{figure}
\begin{center}
\resizebox{4cm}{!}{
  \begin{tikzpicture}

    \node[obs] at (4,1) (Xc) {$\bX$} ;
    \node[dot_node,label={above:$Z_1$}] at (4,3.5) (S1) {} edge[->] (Xc) ;  
     \node[dot_node,label={above:$Z_2$}] at (4.5,3.5) (S2) {} edge[->] (Xc) ;  
      \node[circle,draw=none] at (5.25,3.5) (dots) {$\cdots$}; 
      \node[dot_node,label={above:$Z_\ell$}] at (6,3.5) (Sl) {} edge[->] (Xc) ;

    \node[obs] at (6,1) (Yc) {$Y$} edge[<-] (S1) edge[<-] (S2) edge[<-] (Sl) edge[<-] (Xc)  ;

    \node[anchor=center] at (5,0.8) {$\ba$ }; 
    \node[anchor=center] at (4.5,2) {$M$}; 
   \node[anchor=center] at (5.8,2) {$\bc$}; 
  \end{tikzpicture}
}
\end{center}
\caption{\label{fig:sourcemodel} Model of a confounded influence of $\bX$ on $Y$ where the hidden common causes are independent sources
that influence $\bX$ and $Y$ at the same time. 
}  
\end{figure}
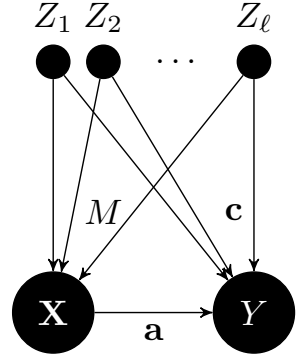
The model induces the following correlations of the observed variables $\bX$ and $Y$:
\begin{eqnarray}
\sx &=& M{\bf I} M^T = MM^T \\
\sxy &=&  MM^T \ba + M \bc,
\end{eqnarray}
where $\one$ denotes the identity matrix. 
While $\ba$ describes the {\it causal influence} of $\bX$ on $Y$, formally regressing $Y$ on $\bX$ yields
\begin{equation}\label{eq:aprime}
\ba':=\sx^{-1} \sxy = \ba + M^{-T} \bc,
\end{equation}
where $M^{-T}$ denotes the transpose of the pseudoinverse of
$M$.  
The vector $\ba'$ describes how the distribution of $Y$ is shifted 
when one {\it observes} that $\bX$ has attained a particular $d$-tuple, while $\ba$ describes how it changes when $\bX$ is {\it set} to some $d$-tuple by an {\it intervention}.
In Pearl's leanguage \cite{Pearl2000}, $\ba'$ vs. $\ba$ 
describe the difference between $p(y|\bx)$ and 
$p(y|do(\bx))$ for our particular linear model. 
\citet{multivariateConfound} define the strength of confounding by
\begin{equation}\label{eq:beta}
\beta := \frac{\|\ba' - \ba\|^2}{\|\ba\|^2 +\|\ba'-\ba\|^2} \in [0,1], 
\end{equation}
which is $0$ for the purely causal case $\ba'=\ba$  and $1$
for the purely confounded case $\ba=0$, which
is already a nice property. 
To further justify this definition, they argue that 
the vectors $\ba$ and $\ba'-\ba=M^{-T}\bc$ are close 
to orthogonal in high dimensions if $\ba$ is drawn independently from $M^{-T}\bc$ from a rotation-invariant distribution. Thus, the denominator is cose to $\|\ba'\|^2$ and $\beta$ is the fraction of squared length of $\ba'$ that can be attributed to the confounder.  
Following \citet{multivariateConfound} we define the estimation of $\beta$ from $P_{\bX,Y}$ as our crucial task. 

The essential assumption that we add now is that the vectors $\bc$ and $\ba$ are randomly drawn from a rotation invariant prior. One can already guess from \eqref{eq:aprime}  what kind of
`non-generic' relation
the vector $\ba'$ then satisfies together with $\sx$: whenever 
$\ba'$ is dominated by the confounding term $M^{-T}\bc$
it tends to be mainly located in the eigenspaces of $\sx$
corresponding to small eigenvalues. The formal analysis is detailed below, but intuitively speaking,  multiplication with $M^{-T}$ amplifies the components corresponding to 
small singular values of $M$ and thus to small eigenvalues of $\sx =MM^T$. 

To formally explore this idea
we first introduce the following generating model for
$\ba$ and $\bc$ and hence for $\ba'$:
\begin{Definition}[ICA based confounding model]
\label{def:gen}
First
sample each component of $\ba$ from a Gaussian with zero mean and standard deviation $\sigma_a$, then sample each component of $\bc$ from a Gaussian with standard deviation $\sigma_c$.
Compute $\ba'$ as in \eqref{eq:aprime}, where $M$
is some given $d\times \ell$-matrix. 
\end{Definition}
Together with $M$, the parameters $\sigma_a$ and $\sigma_c$ determine 
the expected value of $\beta$, but actually only their ratio matters
because $\beta$ depends only on the relative squared lengths of vectors.

\section{Estimating the ratio of $\sigma_a$ and $\sigma_c$}

We now describe how to infer the ratio of $\sigma_a$ and $\sigma_c$
as an intermediate step for inferring $\beta$.
We could infer both parameters 
by maximizing the likelihood of $\ba'$ given 
 our generating model in Definition~\ref{def:gen}
 if we knew $M$ and $\ell$. 
  Unfortunately, we only know $MM^T=\sx$ and $d$.
However, 
 we can construct an equivalent generating model for $\ba'$
 that contains only these observed elements:
\begin{Definition}[alternative generating model for $\ba'$]
\label{def:genOneVec}
Generate $\bb \in \R^d$  by drawing each component from a standard Gaussian.  Set
\[
\ba':= \sqrt{ \sigma^2_a \one + \sigma^2_c \sx^{-1} } \bb. 
\]
\end{Definition}
\begin{Theorem}[equivalence of models]
The model in Definition~\ref{def:genOneVec} generates vectors $\ba'$ with the same distribution as in Definition~\ref{def:gen}. 
\end{Theorem}
\begin{proof} 
First define the $d \times (d+\ell)$-matrix
\[
K_{\sigma_a,\sigma_b} :=\left( \begin{array}{cc} \sigma_a \one &  \sigma_c M^{-T} \end{array} \right).
\]
We can then rewrite $\ba'$ in Definition~\ref{def:gen} as
\[
\ba' = K_{\sigma_a,\sigma_c} \bb',
\]
with
\[
\bb':= \left(\begin{array}{c}\ba /\sigma_a \\ \bc/\sigma_c \end{array}\right).
\] 
Let 
\[
K_{\sigma_a,\sigma_c} = \sqrt{K_{\sigma_a,\sigma_c}
K^T_{\sigma_a,\sigma_c}} V_{\sigma_a,\sigma_c}
\]
be the right polar decomposition of $K_{\sigma_a,\sigma_c}$, where
$V_{\sigma_a,\sigma_c}$ is a partial isometry
from $\R^{d+\ell}$ to $\R^d$. It can be written
as
\[
V_{\sigma_a,\sigma_c} = W_{\sigma_a,\sigma_c} Q,
\]
where $W_{\sigma_a,\sigma_c}$ is an orthogonal
$d\times d$-matrix and $Q: \R^{d+\ell} \to \R^d$ is the projection that annihilates the last $\ell$
components of a vector. We then get
\[
\ba' = \sqrt{K_{\sigma_a,\sigma_c} K^T_{\sigma_a,\sigma_c} } W_{\sigma_a,\sigma_c}
Q \bb'.
\]
Since the $d+\ell$ entries of $\bb'$ are drawn from independent standard Gaussians, the $d$ entries of
$Q \bb'$ are also  standard Gaussians. This  distribution of entries is invariant under orthogonal maps, hence the entries of 
\[
\bb:= W_{\sigma_a,\sigma_c} Q \bb'
\]
are also independent standard Gaussians.
We have
\[
\sqrt{K_{\sigma_a,\sigma_c} K^T_{\sigma_a,\sigma_c}}
= \sqrt{\sigma_a^2 \one + \sigma_c^2 \sx^{-1}},
\]
Hence,
\[
\ba' =  \sqrt{\sigma_a^2 \one + \sigma_c^2 \sx^{-1}} \bb.
\]
\end{proof}
Note that the length of $\ba'$ is irrelevant for $\beta$. We thus consider  $\ba'/\|\ba'\|$ and infer only the quotient $\theta:=\sigma_c^2/\sigma_a^2$. 
 We therefore introduce the matrix
\begin{equation}\label{eq:Stheta}
R_\theta := \one + \theta \sx^{-1},
\end{equation} 
and conclude that our generating models for $\ba'$ induces a distribution for the directions $\ba'/\|\ba'\|$ that is the image of the uniform distribution on the unit sphere (i.e. the Haar measure for the orthogonal group) under the map
\[
\bb \mapsto \frac{\sqrt{R_\theta} \bb}{\|\sqrt{R_\theta} \bb\|}.
\]
To compute this distribution,
we use the following result shown in the appendix:  
\begin{Lemma}[distributions of directions induced by a matrix]\label{lem:Phi}
Let $A$ be an invertible real-valued $d\times d$-matrix. 
Define the map $\Phi: S^{d-1} \rightarrow S^{d-1}$ by 
\[
\Phi(v) := \frac{1}{\|Av\|} Av. 
\]
Then the image of the uniform distribution on $S^{d-1}$ under $\Phi$
has the following density with respect to the uniform distribution:
\begin{equation}\label{eq:density}
p(\tilde{v})= \frac{1}{\det(A) \|A^{-1}\tilde{v}\|^d}. 
\end{equation}
\end{Lemma}
We now apply Lemma~\ref{lem:Phi} to $A:=\sqrt{R_\theta}$
as defined by \eqref{eq:Stheta} and 
obtain
\begin{equation}\label{eq:ptheta}
p_\theta (\tilde{v}) = \frac{1}{|\det \sqrt{R_\theta}| \left\|\sqrt{R^{-1}_\theta} \tilde{v}\right\|^d}.
\end{equation}
Using
\[
|\det \sqrt{R_\theta}| = \sqrt{ \det R_\theta }
\] 
we can rewrite \eqref{eq:ptheta} as 
\[
p_\theta  (\tilde{v}) = \frac{1}{\sqrt{\det R_\theta}
\| \langle \tilde{v}, (1+\theta\sx^{-1})^{-1} \tilde{v} \rangle \|^{d/2}},
\]
which proves the following theorem:
\begin{Theorem}[density of directions]
The generating model in Definition~\ref{def:gen} generates vectors $\ba'$ whose distribution of unit vectors $\tilde{v}:=\ba'/\|\ba'\|$ 
has the following log density with respect to the uniform distribution on the sphere:
\begin{eqnarray}
&&\log p_\theta (\tilde{v}) \nonumber \\
&=& 
\frac{1}{2}\left[ \log \det R_\theta  - d \log \langle \tilde{v}, R_\theta^{-1} \tilde{v}\rangle \right]. \label{eq:likeli}
\end{eqnarray}
\end{Theorem}
Given sufficiently many samples $\ba'$ generated with the same $\theta$, we can certainly infer $\theta$ by maximizing \eqref{eq:likeli}. Remarkably, 
we can infer the loglikelihood already from a single instance for large $d$ under appropriate conditions:
\begin{Theorem}[concentration of measure]\label{thm:conc}
Let $\tilde{v}$ be drawn from $p_{\theta'}$. Then for sufficiently small $\epsilon$ we have 
\[
\left|\log p_\theta (\tilde{v}) - 
\frac{1}{2}\left[\log \det R_\theta- \log \frac{\tau (R_{\theta'} R^{-1}_{\theta})}{ \tau (R_\theta')}  \right] \right|   \leq \epsilon
\]
with probability at least
\[
1 - \frac{1}{d \epsilon^2} \left( \frac{\tau (R_\theta^2 R_{\theta'}^{-2})}{\tau (R_\theta R_{\theta'})^2} 
+  \frac{\tau (R_{\theta'}^{2})}{\tau (R_{\theta'})^2} \right), 
\]
 where $\tau ():=\frac{1}{d}{\rm tr} (.)$ denotes the
renormalized trace. 
\end{Theorem}
The proof can be found in the appendix. 
Whenever one assumes a limit for $d\to \infty$ in which
the expressions with $\tau$ converge\footnote{This holds, for instance, for any sequence  $\sx^{(d)}$ for which
the eigenvalues have a uniform positive lower bound $b$
and the
distribution of eigenvalues converges weakly to some measure $\mu$. Then, $\tau\left( f(\sx^{(d)})\right)$  converges to $\int f d\mu$ for any bounded continuous function $f:[b,\infty)\rightarrow \R$ by definition of weak convergence.}, the error thus tends to zero. 
Intuitively speaking, the reason is that drawing one vector 
from $p_{\theta'}$ in dimension $d$ can be reduced to drawing $d$ independent coefficients with respect to an appropriate basis, which finally reduces the problem to the usual law of large numbers.

\section{Estimating confounding\\ strength $\beta$ \label{sec:estbeta}}

To infer $\beta$ (which we defined as our crucial task)
 from $\theta$ we need some approximations that hold for
large $d$. First we use $\|\ba\|^2/d \approx \sigma_a^2$
which is justified by the law of large numbers. Moreover we can estimate the length of $M^{-T} \bc$ using the trace of the concentration matrix of $\bX$:
\begin{eqnarray*}
\frac{1}{d}\|\ba' -\ba\|^2 &=& \frac{1}{d} \|M^{-T} \bc\|^2 = \frac{1}{d} \langle \bc, M^{-1}M^{-T} \bc\rangle \\ &\approx & 
\sigma_c \tau(M^{-1} M^{-T}) \\
&=&
\sigma_c \tau(M^{-T} M^{-1}) \\&=& \sigma_c \tau(\sx^{-1}), 
\end{eqnarray*}
where the approximation uses also the law of large numbers since we can generate $\bc$ by drawing its coefficients with respect to
the eigenbasis of $M^{-1} M^{-T}$ from independent Gaussians of
standard deviation $\sigma_c$. 
Thus we obtain
\begin{equation}\label{eq:betaApp}
\beta \approx \frac{\tau(\sx^{-1} ) \sigma_c^2}{\tau (\sx^{-1} ) \sigma_c^2 +  \sigma_a^2}= \frac{\tau (\sx^{-1} ) \theta}{\tau (\sx^{-1} ) \theta + 1}. 
\end{equation}

Putting everything together, we obtain the following procedure for estimating $\beta$ from $(\bX,Y)$ samples:
\begin{enumerate}
\item Compute the empirical covariance matrices
$\widehat{\sx}$ and $\widehat{\sxy}$.

\item Estimate $\ba'$ via 
\[
\widehat{\ba'} := \widehat{\sx}^{-1} \widehat{\sxy}.
\]

\item Infer $\theta$ via maximizing the likelihood
$
\log p_\theta (\widehat{\ba'}/\|\widehat{\ba'}\|)
$
defined by \eqref{eq:likeli}.

\item Compute $\beta$ from the estimated value of $\theta$ via
\eqref{eq:betaApp}.

\end{enumerate}

Here we have neglected finite sample issues completely. 
We will discuss them in section~\ref{sec:overfitting}.

\section{Test for non-confounding \label{sec:test}}

To test the null hypothesis  $\theta=0$, that is $\ba'=\ba$, we
define the test statistics (applied to a {\it single} instance $\tilde{v}=\ba'/\|\ba'\|$)
\begin{equation}\label{eq:T}
T(\tilde{v}) := \frac{1}{\sqrt{d}} \left\{\langle \tilde{v},\sx^{-1} \tilde{v} \rangle -
\tau (\sx^{-1}) \right\}.
\end{equation}
One can easily show that its expectation is zero when $\tilde{v}$
is drawn uniformly at random from the unit sphere, which we assumed for the unconfounded case. Intuitively, the definition of $T$ is motivated by the idea to detect overpopulation of
eigenspaces with small eigenvalues, which we expect for confounding. As a further justification, we observed that
$T$ coincides, up to a scaling factor, with the score function 
\[
\frac{\partial \log p_\theta(\tilde{v})}{\partial \theta},
\]
at $\theta=0$.
This is a  natural candidate for detecting changes of $\theta$ because score functions occur in the construction of optimal estimators whenever there  exist unbiased estimators attaining the Cram\'{e}r Rao bound \cite{Cramer46}.

To derive a simple approximation for the null distribution of $T$
we think of $\tilde{v}=\ba'/\|\ba'|$ as being generated by drawing its coefficients $a_j$ 
with respect to the eigenbasis of $\sx^{-1}$ from $\cN(0,1/\sqrt{d})$ followed by renormalization:
\begin{eqnarray*}
T(\tilde{v}) &= & \frac{1}{\sqrt{d}} \left( \frac{\sum_{j=1}^d a^2_j s_j }{\sum_{j=1}^d s^2_j}
- \tau (\sx^{-1}) \right)\\
&\approx & \frac{1}{\sqrt{d}} \left(\sum_{j=1}^d a^2_j s_j 
- \tau (\sx^{-1}) \right),
\end{eqnarray*}
where $s_j$ denotes the eigenvalues of $sx^{-1}$.
Already for moderate size of $d$, we can thus get a good approximation for the null distribution of $T$ 
by a weighted sum of squared Gaussian, i.e., it
approximately follows
a mixed $\chi^2$-distribution. 

\section{Overfitting\label{sec:overfitting}} 

So far we have completely ignored finite sampling issues.
High-dimensional regression requires regularization which could
spoil our model assumptions, e.g., if they enforce sparsity
which is not compatible with our rotation invariant prior on $\ba$. 
Therefore, the method should only be applied if the sample size is sufficiently high for the respective dimension (see 
section~\ref{sec:exsim}) to avoid overfitting. Remarkably,
overfitting generates the sample kind of `dependences' between
the estimator of $\ba'$ and the estimator of
$\sx$ as confounding generated for the true objects $\ba'$ and $\sx$ themselves. 

To show this,
assume that $Y$ is independent of $\bX$
and let $(x^j_1,\dots,x^j_d,y^j)$ for $j=1,\dots,n$ be samples independently drawn from $P_{\bX} P_Y$, where $P_\bX$ is arbitrary and $P_Y$ is Gaussian.
Define the matrix 
\[
\bx:=\left(x^i_j-
\bar{x}_j \right)_{i=1,\dots,n,j=1,\dots,d},
\]
where $\bar{x}_j:=\frac{1}{n} \sum_{i=1}^n x^i_j$ denotes the empirical average of the respective component.
Likewise, 
define the vector $y:=(y^1,\dots,y^n)^T - \bar{y} (1,\dots,1)^T $. 
Then, we obtain 
$
\sx = \bx^T \bx 
$
and 
$
\sxy = \bx^T y
$
(where we have skipped the symbol $\widehat{\cdot}$
for better readability). 
Since $y$ is the projection of $(y_1,\dots,y_n)^T$ onto the orthogonal complement of ${\bf 1}:=(1,\dots,1)^T$, its distribution is isotropic
in the $n-1$-dimensional subspace defined by the orthogonal complement ${\bf 1}^\perp$ of ${\bf 1}$. Let $V$ be an $(n-1)\times n$ matrix that rotates 
${\bf 1}^\perp$ onto $\R^{n-1}$. Then we may write
\[
\sx = \bx^T V^T V \bx, 
\]
because the image of $\bx$ is contained in the image of the projection
$V^T V$.
Moreover,
\[
\sxy = \bx^T V^T V y.
\]   
To show the formal analogy to the mixing scenario above we now set $M:= V\bx$ and $y':=V y$. Then we can write
$
\sx = M^T M
$
and
$
\sxy =M^T y',
$ 
and thus obtain
\[
\hat{\ba} = M^{-T} y',
\]
where $y'$ is isotropically chosen from $R^{n-1}$.
The generating model for $\hat{\ba}$ thus coincides
with the model above with $\ell = n-1$ for the case of pure confounding. 

Computing an unregularized regression 
for $\bX$ and $Y$ being independent
thus yields 
a regression vector $\ba'$ that is also mainly located in the low eigenvalue eigenspace of $\sx$. 
We expect the same behavior if $\bX$ influences $Y$ 
without confounder when the sample size is so small that the observed correlations are dominated by
statistical fluctuations rather than by the true causal influence. 

On the one hand one may regret that confounding and overfitting becomes indistinguishable. On the other hand, the method thus provides an unified approach to detect that a regression vector $\widehat{\ba'}$  does not show the true causal influence; either because $\widehat{\ba'} \neq \ba'$ or because
$\ba' \neq \ba$ due to confounding. 
There is a simple reason why both cases generate similar dependences between $\sx$ and $\ba'$:
Whenever $\sxy$ is a vector that has been generated independently of $\sx$, the vector $\sx^{-1}\sxy$   
tends to live mainly in the small eigenvalue subspace 
of $\sx$. Only if $\sxy$ is not drawn independently of $\sx$, for instance, because it is generated by
$\sx \ba$ (where $\ba$ is drawn independently of $\sx$),  this overpopulation of small eigenvalues does not happen.

\section{Experiments with simulated data \label{sec:exsim}}

The code and the data sets for all experiments are available at \url{http://webdav.tuebingen.mpg.de/causality/}.
We generated models as follows:
\begin{enumerate}
\item We have drawn $n$ samples of each $Z_1,\dots,Z_\ell$ as independent standard Gaussians 
\item We have drawn the entries of $M$ by independent standard Gaussians
\item We have drawn the parameters $\sigma_a,\sigma_c$ from the uniform distribution on $[0,1]$
\item We have drawn each coefficient of $\ba\in \R^d$ and $\bc\in \R^\ell$ from Gaussians of standard deviation
$\sigma_a$ and $\sigma_c$, respectively.  
\item We computed samples $(\bX,Y)$ via the structural equations 
$\bX = M \bZ$ and $Y= \ba^T \bX +  \bc^T \bZ$. 
\end{enumerate}
Knowing the above  parameters, we can easily compute the exact confounding strength using
\[
\beta = \frac{\|M^{-T} \bc\|^2}{\|\ba\|^2 + \|M^{-T} \bc\|^2}.  
\]

\subsection{Estimating $\beta$}

We have estimated $\beta$ as described at the end of section~\ref{sec:estbeta} for $d=\ell=10,20,50,100$ with sample size $10,000$. The scatter plots in Figure~\ref{fig:allsimulations}
show the relation between the true values $\beta$ and the estimated values $\hat{\beta}$.
\begin{figure}
\centerline{
\begin{tabular}{c}
\adjincludegraphics[width=0.3\textwidth,trim={0.08\width}]{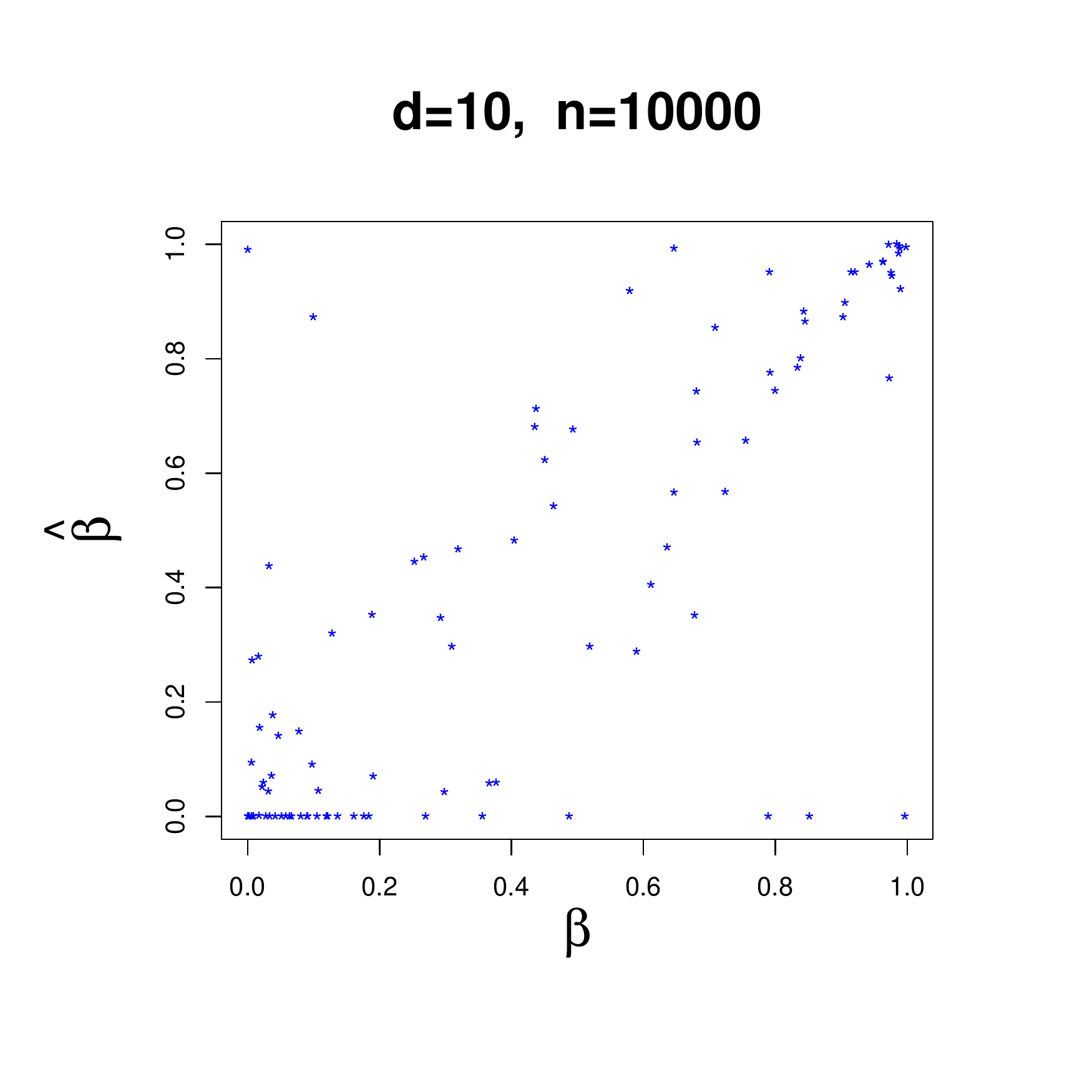}   \\
\adjincludegraphics[width=0.3\textwidth,trim={0.08\width}]{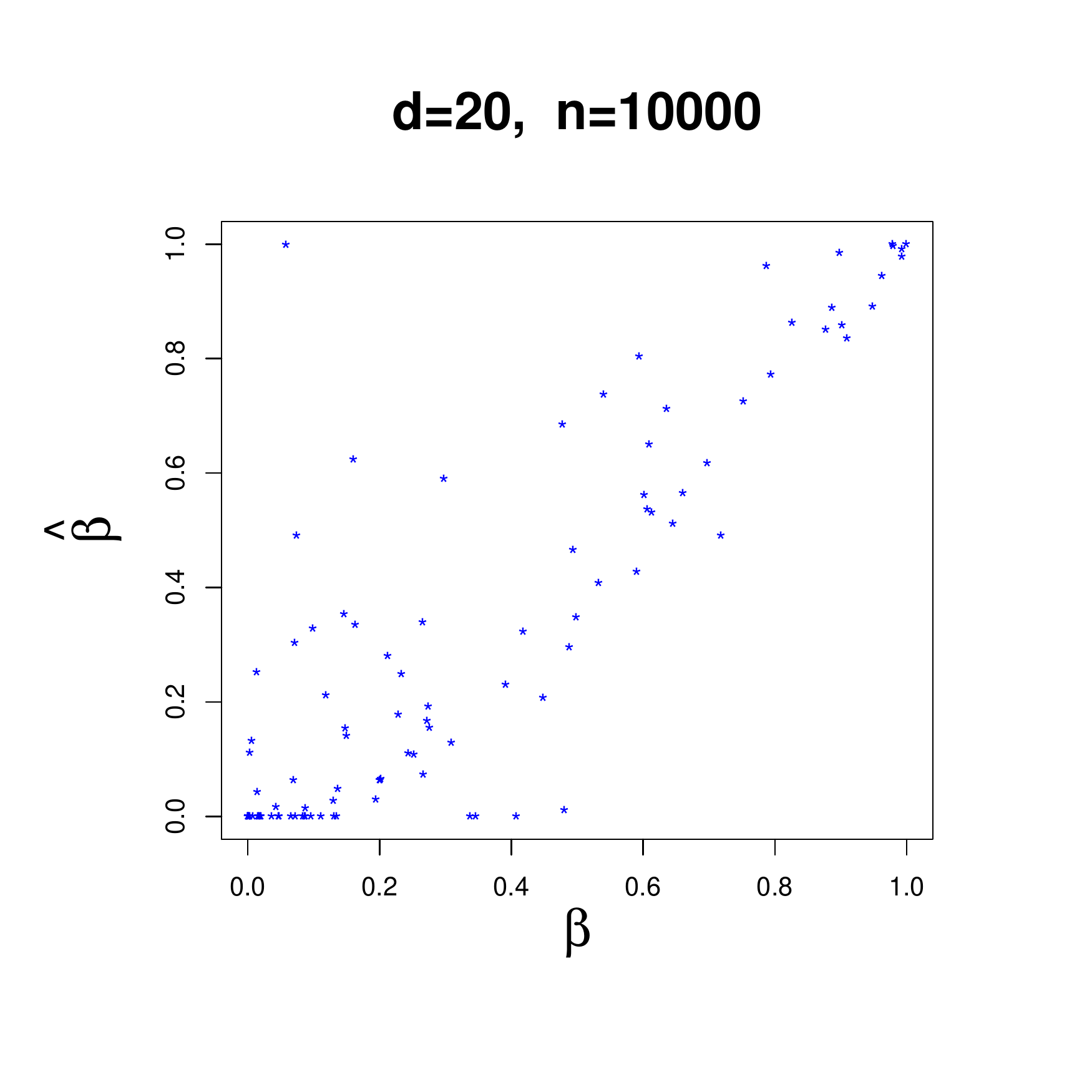} \\
\adjincludegraphics[width=0.3\textwidth,trim={0.08\width}]{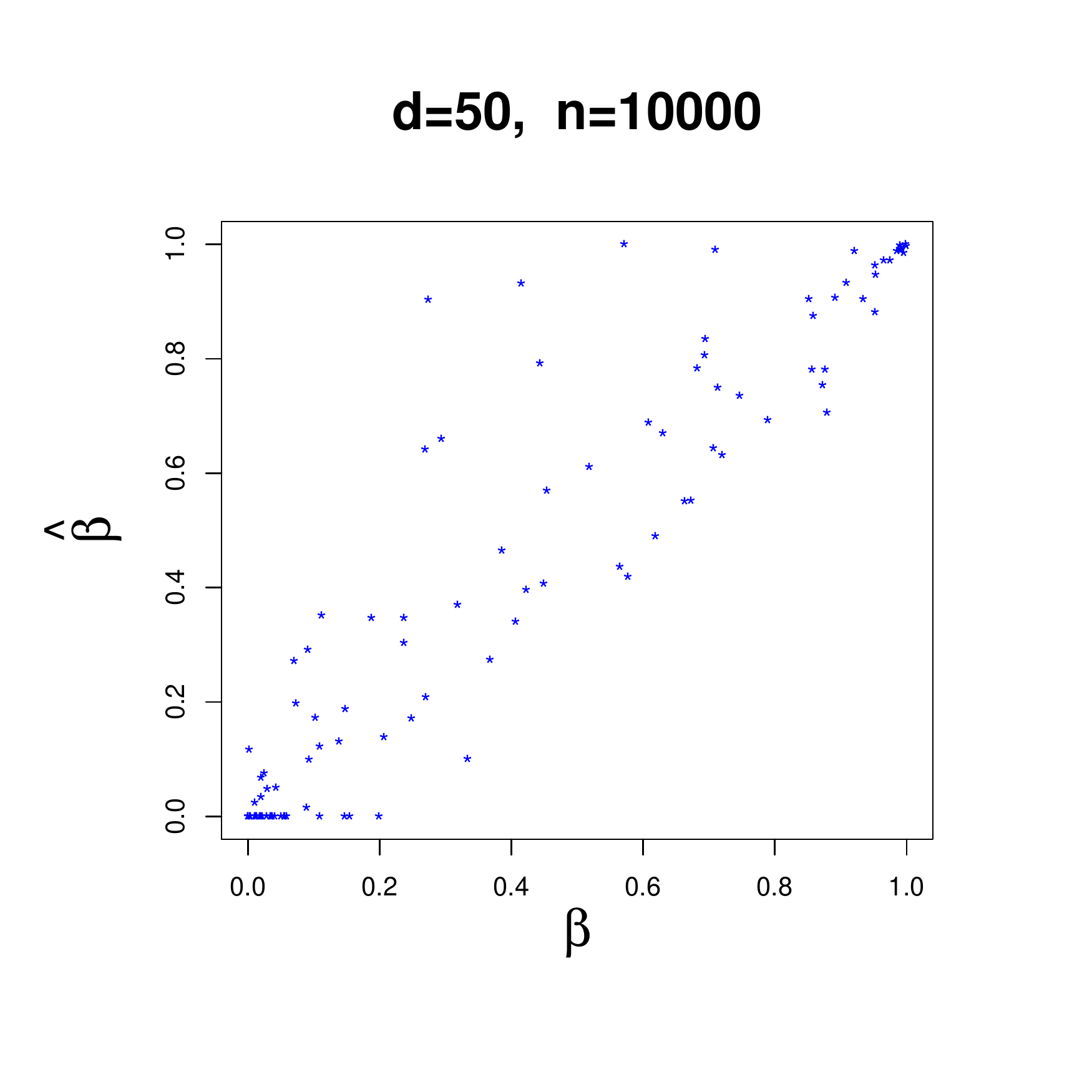}  \\
\adjincludegraphics[width=0.3\textwidth,trim={0.08\width}]{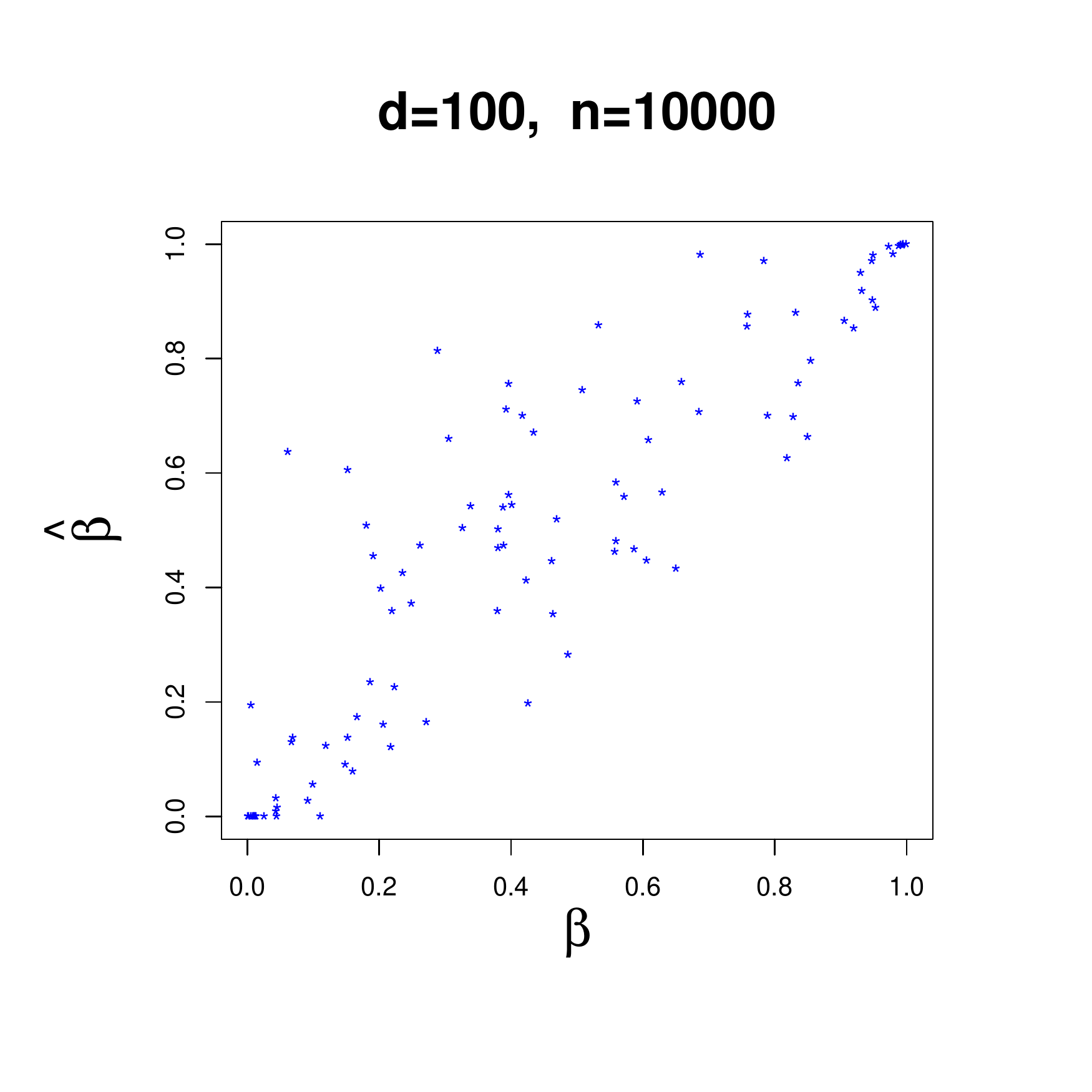}
\end{tabular}
}
\caption{\label{fig:allsimulations} Simulation results: true value $\beta$ versus estimated 
value $\hat{\beta}$ for different dimensions $d$ and sample size $n=10,000$.} 
\end{figure}
One can see that $\beta$ and $\hat{\beta}$ are clearly correlated
and that the performance increases (although slowly) for higher dimension. The estimation is reasonably good in the regions
where $\beta$ is close to $0$ or $1$, which suggests that one should rather trust in the qualitative statement about whether there is confounding or not than in the exact value of $\hat{\beta}$.

Since our theory has shown that $\ell$ is completely irrelevant in our idealized scenario 
provided that it is not smaller than $d$
(see the generating model in Definition~\ref{def:genOneVec}) it would be pointless  to explore the case $\ell>d$ here.

\subsection{Test for non-confounding}

For the simulated data described above we have applied the 
test for unconfoundedness described in section~\ref{sec:test}
by drawing $1000$ samples from the null distribution of $T$ and comparing them to the observed value $T(\hat{\ba'}/\|\hat{\ba'}\|)$. Figure~\ref{fig:testScatter} visualizes the joint distribution of p-values with $\beta$. 

One can see that for $\beta>0.5$ the p-values begin to
be mostly close to zero. 
Figure~\ref{fig:rejections} 
shows how the fraction of rejections increases when
$\beta$ increases for the two cases where the confidence level $\alpha$ is set to $0.1$ (left) or $0.05$ (right).
Here we have chosen a one-sided test because confounding increases $T$ due to the overpopulation of subspaces with small eigenvalues of $\sx$.  
The results show that for those confidence levels 
unconfoundedness is mostly rejected for models with $\beta>0.6$.

\begin{figure}[!ht]
\centerline{
\adjincludegraphics[width=0.5\textwidth,Clip={0\width} {0.15\height} {0\width} {0.05\height}]{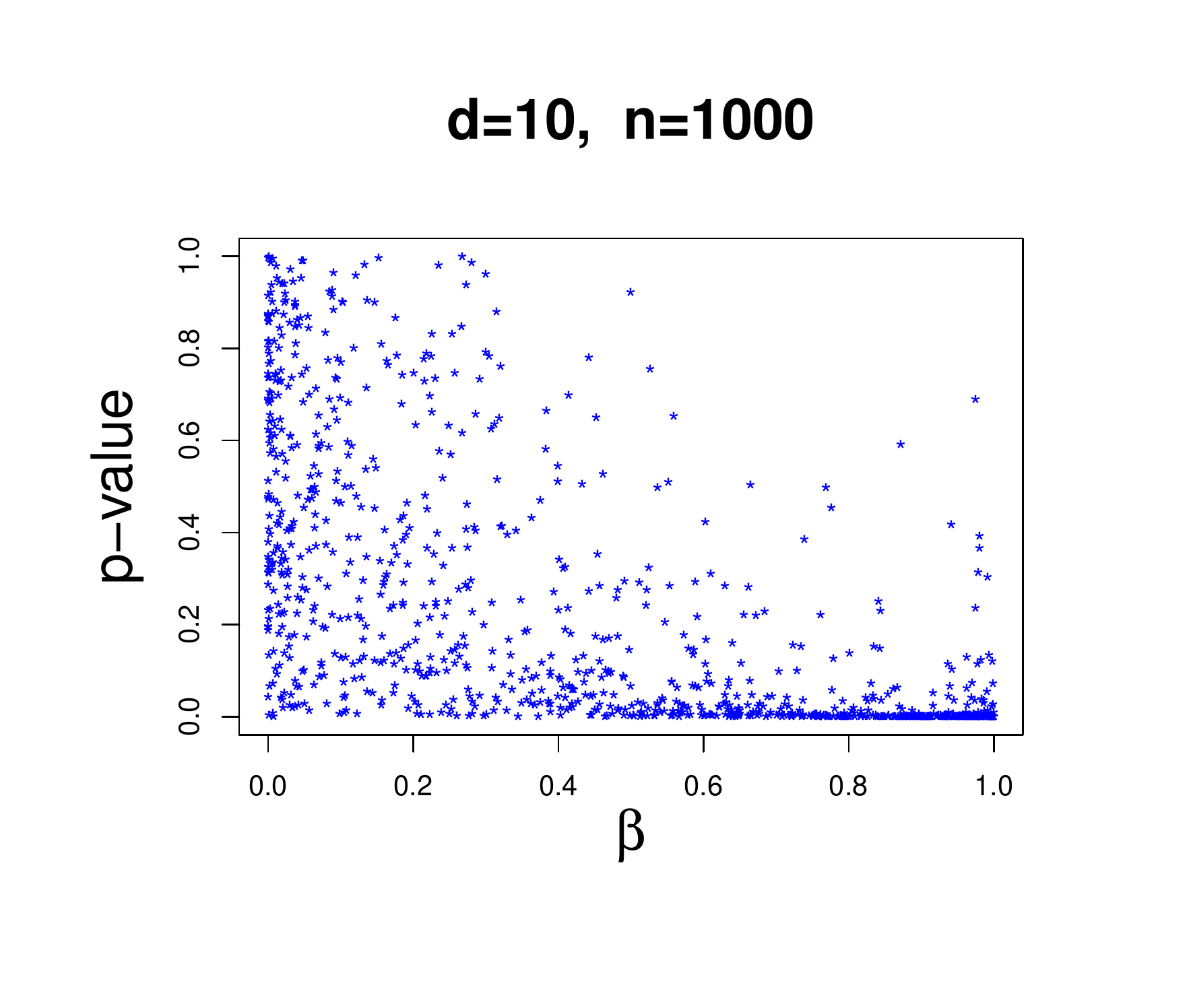}
}
\caption{\label{fig:testScatter} p-values obtained in the test for non-confounding for different values of the confounding parameters $\beta$. It can be seen that the p-values get close to zero when  $\beta$ tends to $1$.} 
\end{figure}
\begin{figure}[t]
\centerline{
\begin{tabular}{c}
\adjincludegraphics[width=0.35\textwidth,Clip={0\width} {0.08\height} {0\width} {0.06\height}]{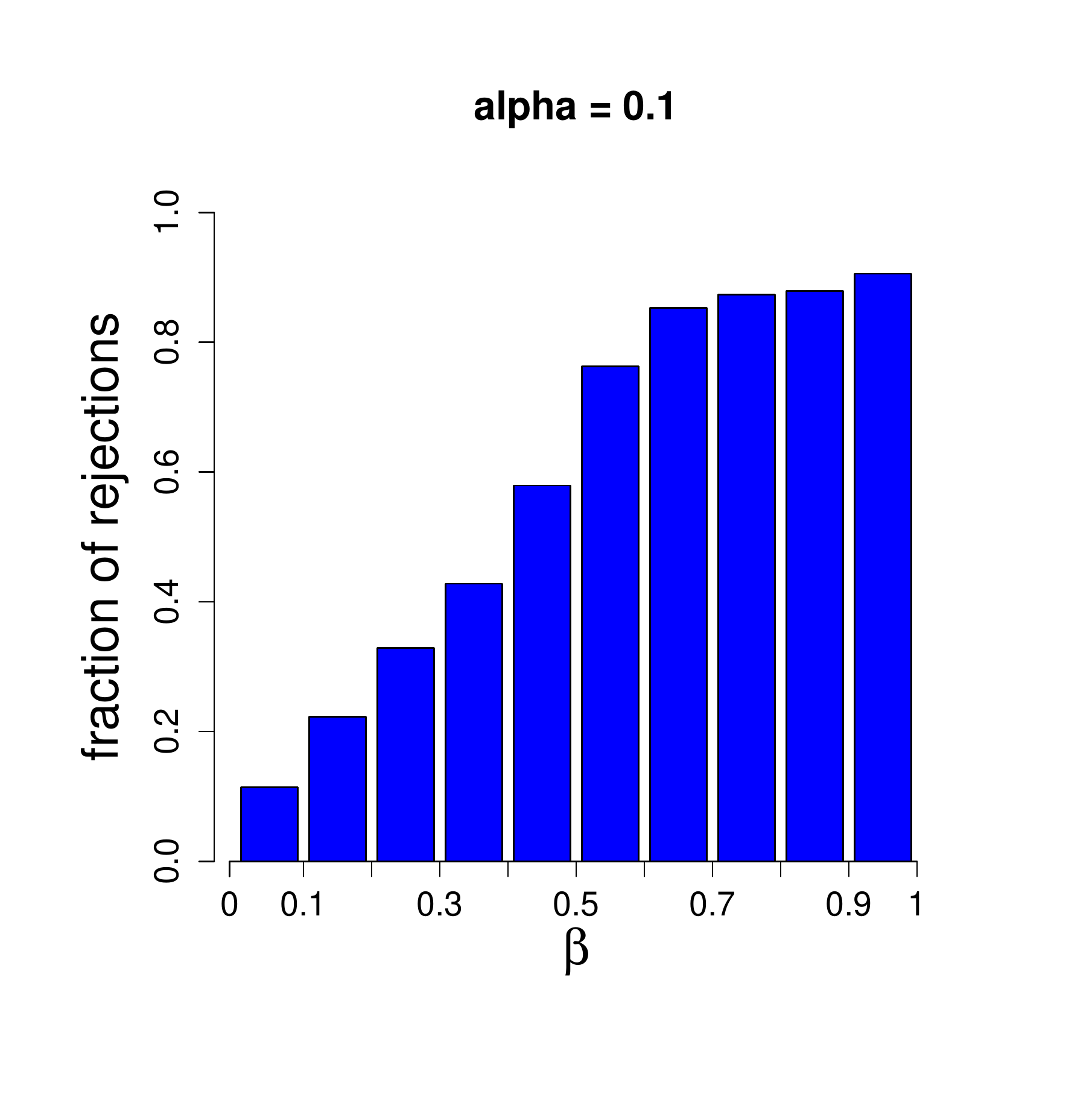}\\
\adjincludegraphics[width=0.35\textwidth,Clip={0\width} {0.08\height} {0\width} {0.06\height}]{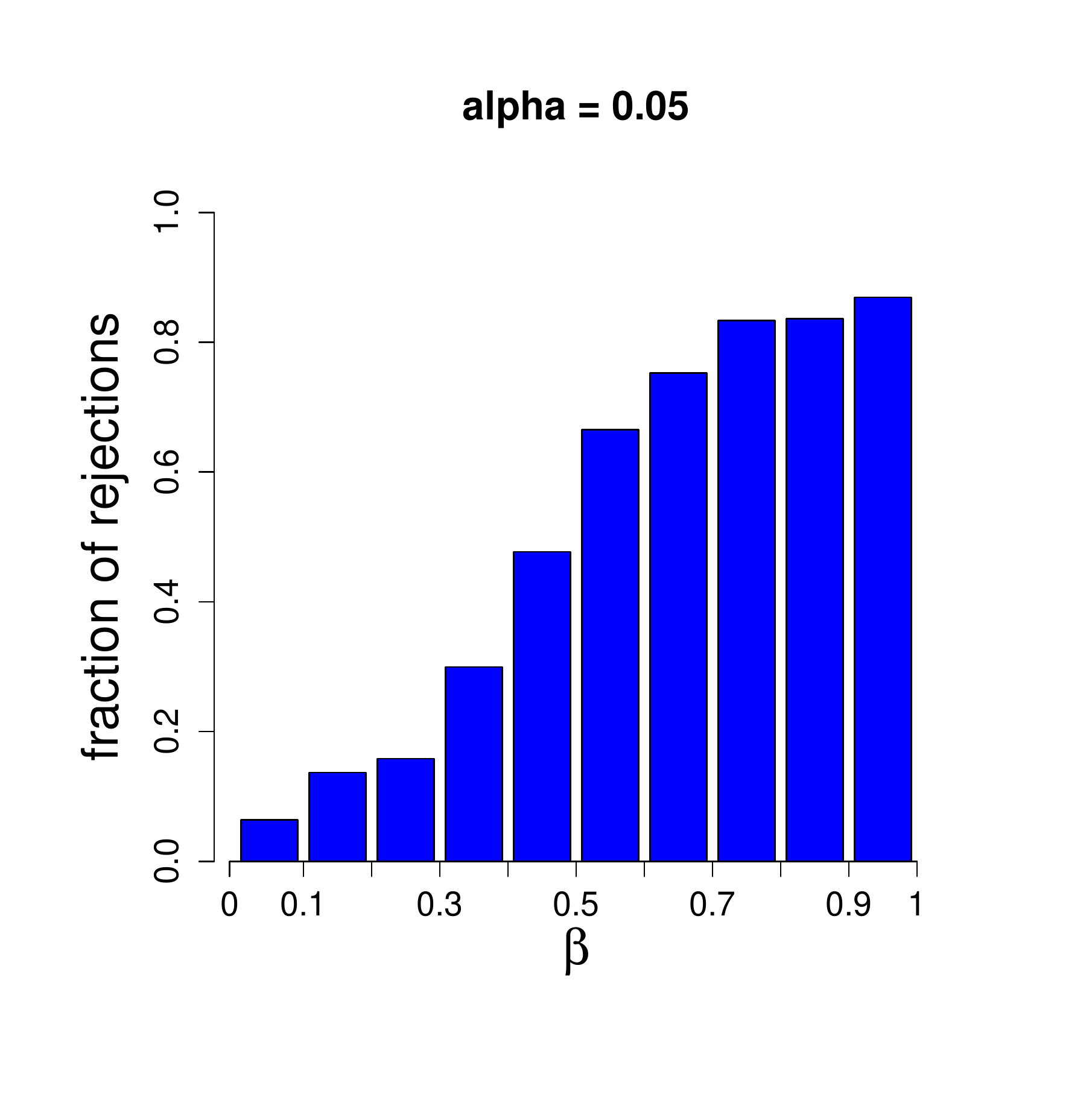}
\end{tabular}
}
\caption{\label{fig:rejections} Fraction of rejections when the confidence is chosen to be $0.1$ (top) and $0.05$ (bottom) for different values of $\beta$ in 5000 runs.} 
\end{figure}

\begin{figure}[!ht]
\centerline{
\begin{tabular}{c}
\adjincludegraphics[width=0.25\textwidth,Clip={0\width} {0.03\height} {0\width} {0.06\height}]{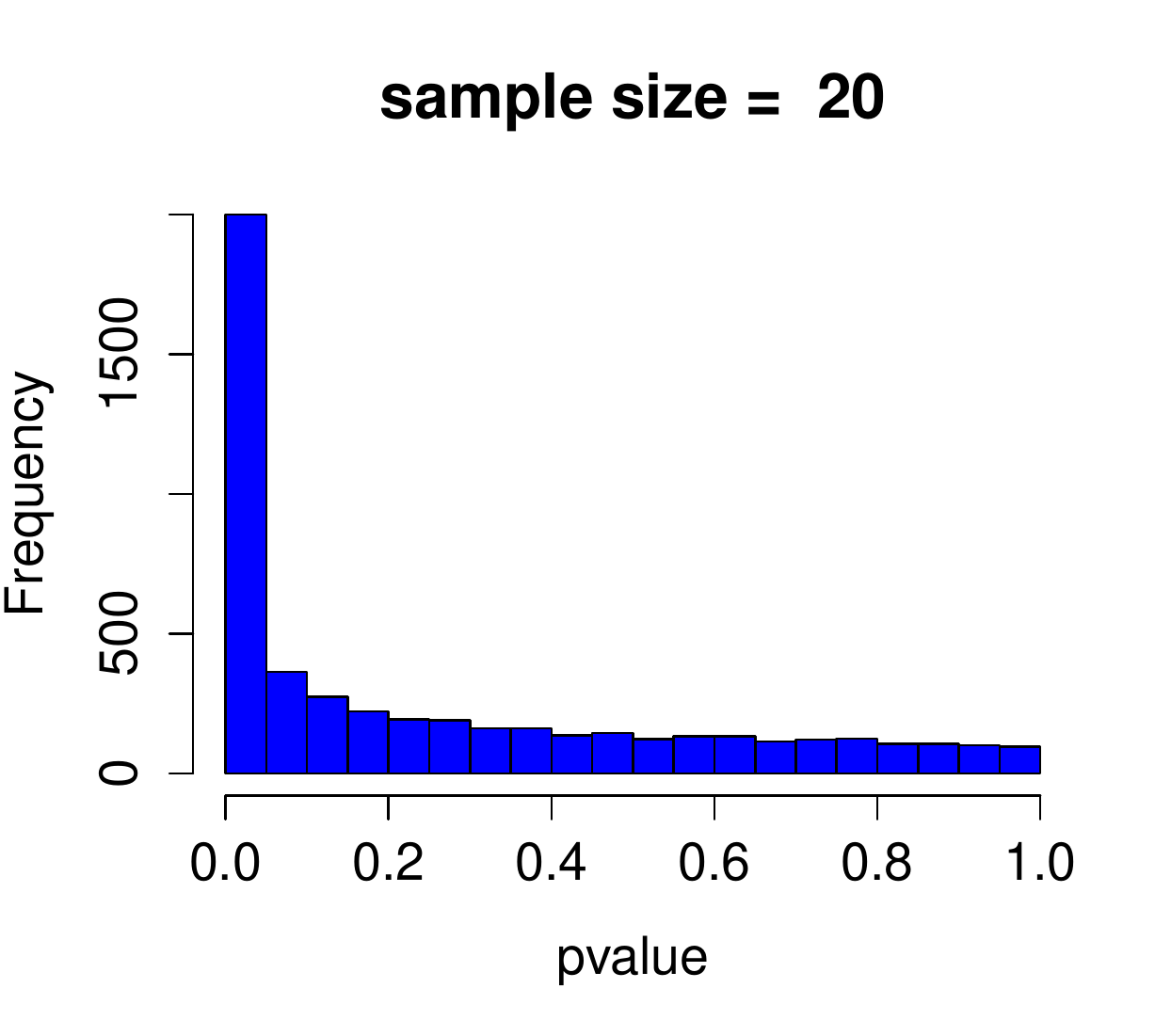} \\
\adjincludegraphics[width=0.25\textwidth,Clip={0\width} {0.03\height} {0\width} {0.06\height}]{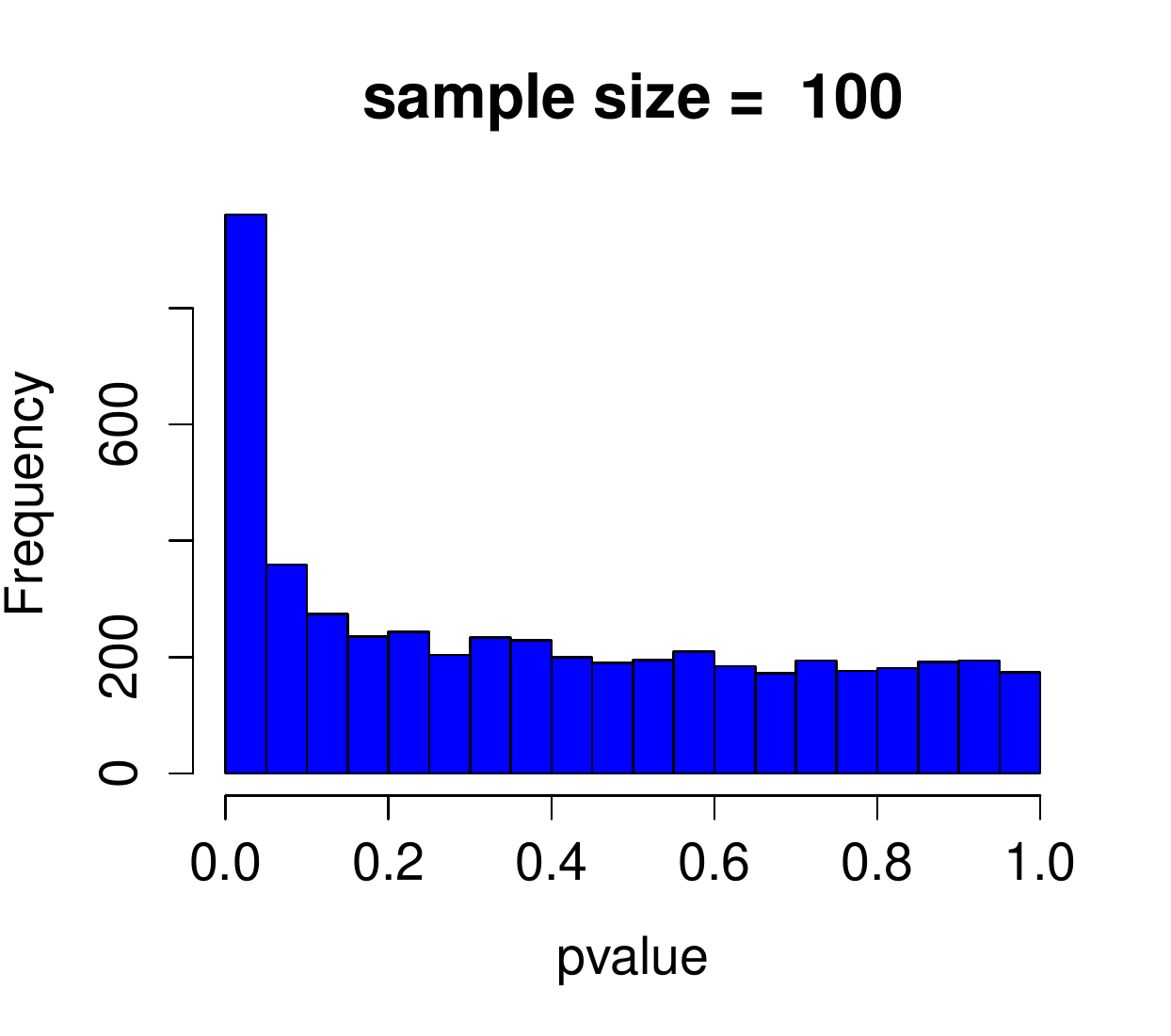}\\
\adjincludegraphics[width=0.25\textwidth,Clip={0\width} {0.03\height} {0\width} {0.06\height}]{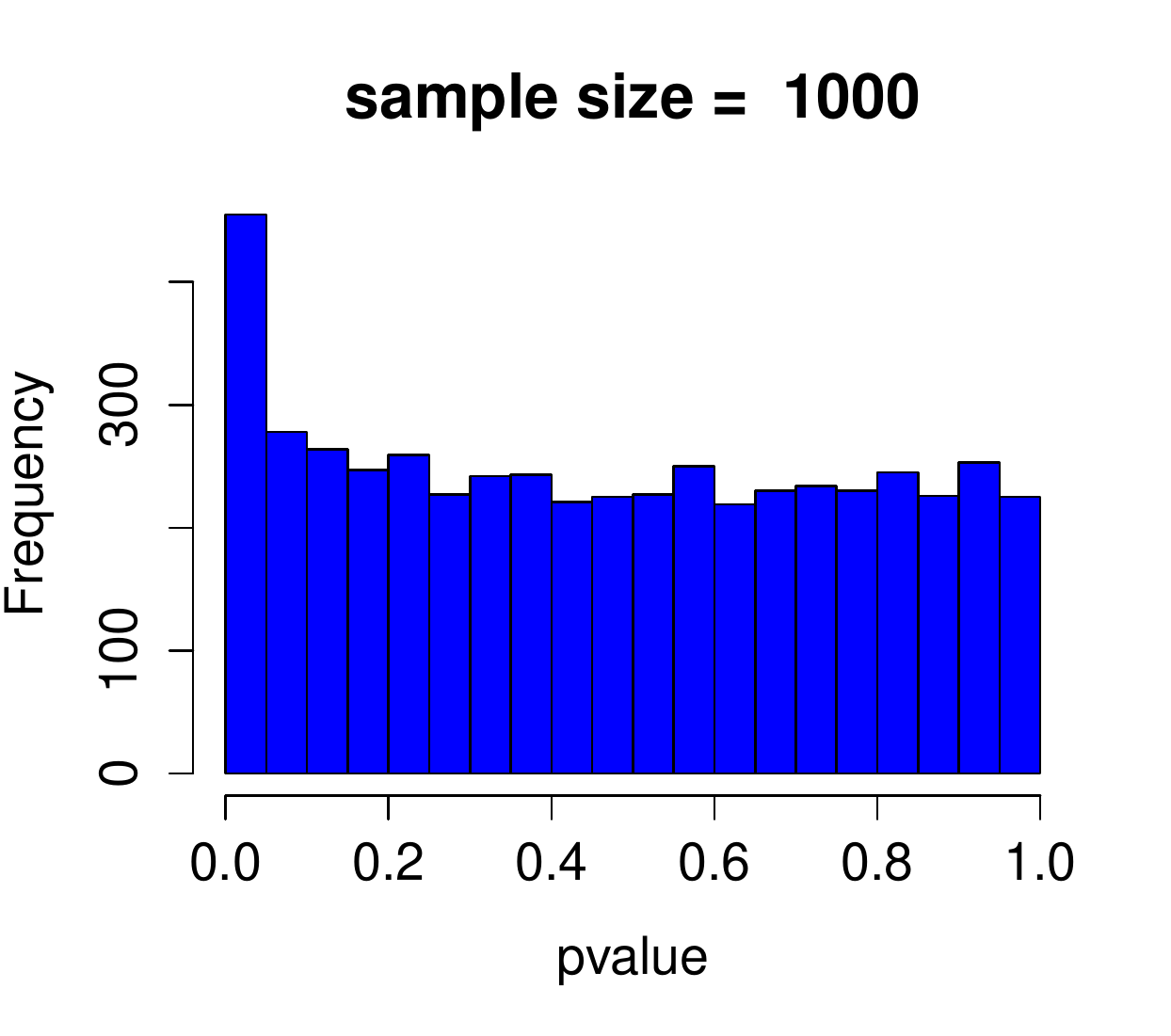} \\
\adjincludegraphics[width=0.25\textwidth,Clip={0\width} {0.03\height} {0\width} {0.06\height}]{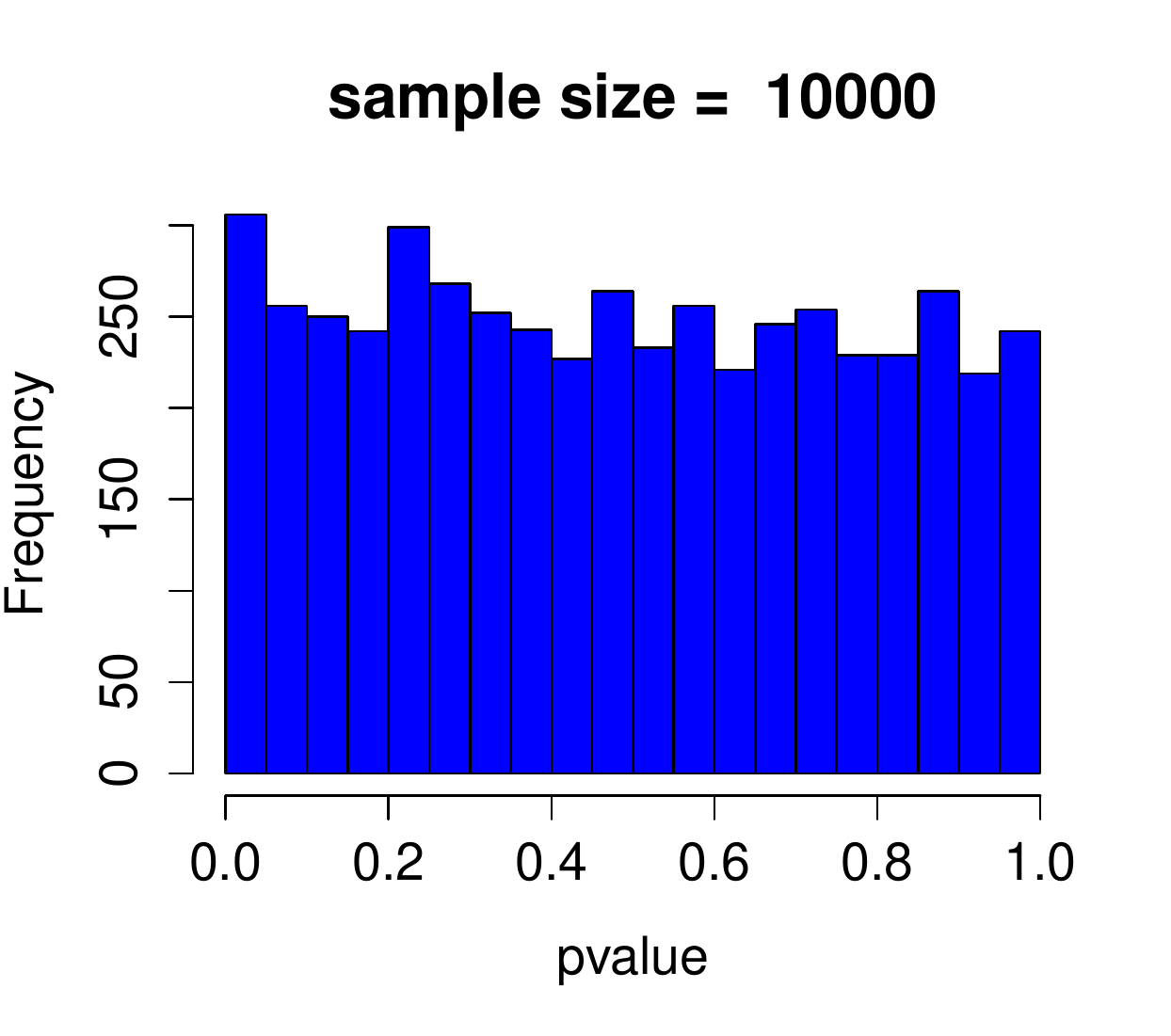}
\end{tabular}
}
\caption{\label{fig:overfit} Distribution of p-values in the statistical test for unconfoundedness in a scenario without confounding.} 
\end{figure}

\subsection{Overfitting} 

We generated $d+1$-tuples of $\bX,Y$  by  
first drawing $\bX$ via a random mixing matrix and then $Y$ by $Y = \ba^T \bX + E$, where $E$ is $\cN(0,1)$ distributed and $\ba$ is a
random vector whose entries are randomly drawn from $\cN(0,1)$. 

Figure~\ref{fig:overfit} shows the distribution of p-values of the test for unconfoundedness for different sample sizes $n$.
As one can see, for $n=20$ one gets mostly small p-values  although the model is actually unconfounded
(in agreement with our theoretical insights saying that overfitting yields the same type of untypical 
regression vectors as confounding). For $n=100$ and $n=1000$, small p-values are still overrepresented
and only 
for $n=10,000$ the distribution of p-values is close to uniform. 
This suggests that dimension $10$ already requires sample sizes of the order $10,000$ if one wants to avoid too many false rejections (when focusing on confounding rather than on overfitting).

\section{Experiments with real data}

Since it is hard to get data where the confounding strength $\beta$ is known we can mostly only discuss plausibility except
for the data set in the following section.

\subsection{Data from an optical device}

\citet{multivariateConfound} describe 
an optical device where the causal structure
and  $\beta$ is known by construction. The variable
$\bX$ is a low-resolution image ($3\times 3$ pixel)
shown on the screen of a laptop and $Y$ is the brightness measured by a photodiode at some distance in front of the screen. 
The image $\bX$ is generated by a webcam placed in front of a TV.
As confounder $Z$ (which is one-dimensional following the assumptions of \citet{multivariateConfound}), an LED in front of the photodiode and another LED in front of the webcam is controlled by a random noise. Since $Z$ is known,
an approximation of $\beta'$ for
 $\beta$ 
can be directly computed from the observed covariances 
($\beta\neq \beta'$ only due to finite sample issues).
We first tried the $11$ data sets
 with variable confounding and obtained the results displayed in Figure~\ref{fig:opticalResults}.
The results are quite similar to those from \citet{multivariateConfound} although the scenario matches the
very specific one-dimensional confounding scenario there while our model is more general. 
Also here the results are qualitatively right 
($\beta'$ and $\hat{\beta}$ 
are significantly correlated)
but with a clear tendency to underestimate confounding, which has alerady been discussed by \citet{multivariateConfound}.

\begin{figure}[ht]
\centerline{
\adjincludegraphics[width=0.4\textwidth,Clip={0\width} {0.13\height} {0\width} {0.1\height} {0\width}]{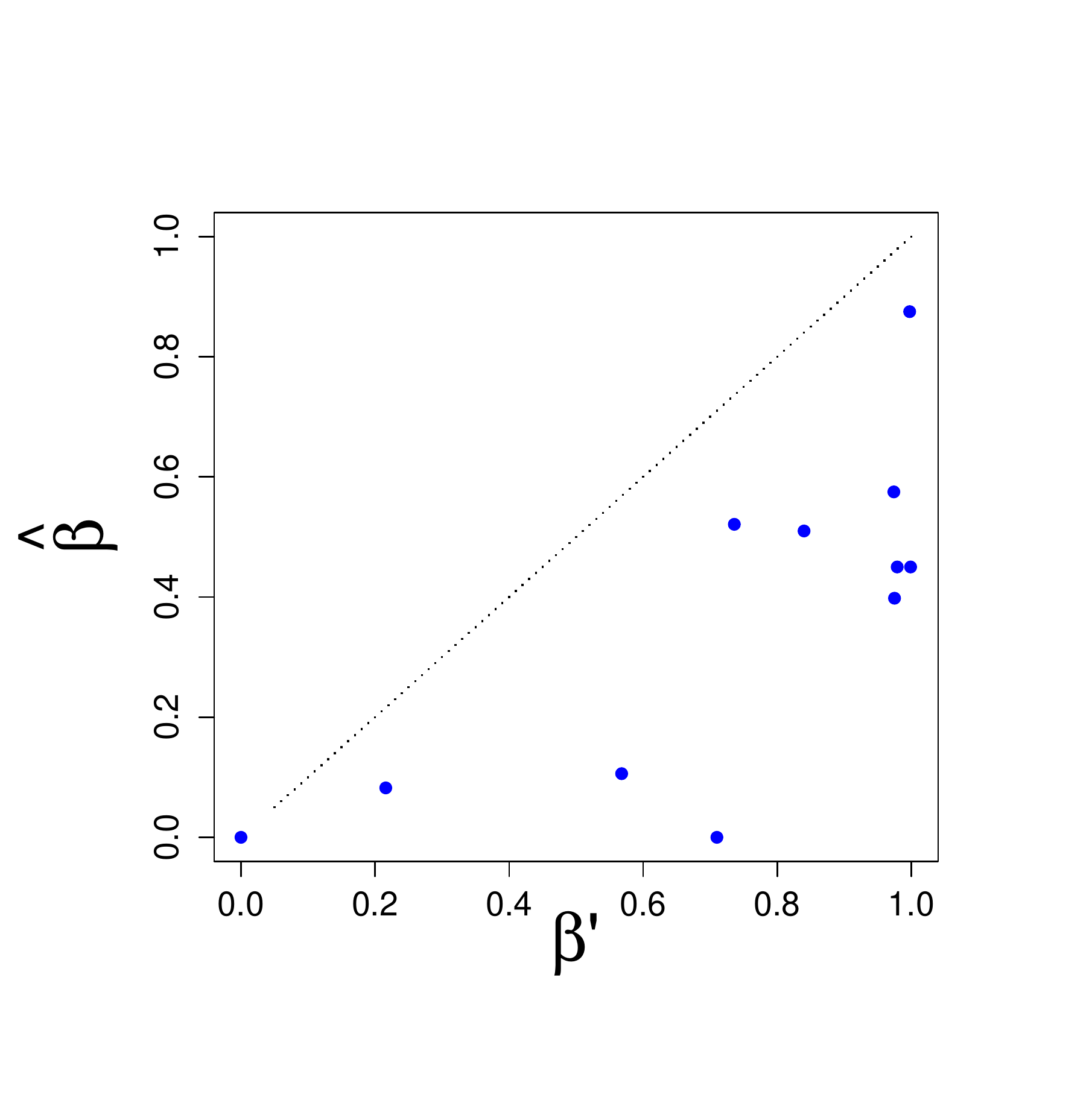} 
}
\caption{\label{fig:opticalResults} 
True and estimated confounding strength for the optical device used by \citet{multivariateConfound}.} 
\end{figure}

We also tested the two data sets where one is purely confounded  ($\beta=1$) and one completely unconfounded ($\beta=0$) and obtained
$\hat{\beta}=0.768$ and
$\hat{\beta}= 0$, respectively.

\subsection{Taste of wine}

This dataset  \cite{Lichman2013} describes the dependence between 
the scores on the taste between 0 and 10  (given by human subjects) of red wine, and 11 different ingredients:
$X_1$: fixed acidity, 
$X_2$: volatile acidity, 
$X_3$: citric acid, 
$X_4$: residual sugar, 
$X_5$: chlorides, 
$X_6$: free sulfur dioxide, 
$X_7$: total sulfur dioxide, 
$X_8$: density, 
$X_9$: pH, 
$X_{10}$: sulphates, 
$X_{11}$: alcohol.  
Taking the taste $Y$ as target variable
we obtained $\hat{\beta}=0$ (after we normalized all $X_j$ to unit variance since their scale where incompatible) which is plausible to some extent given
that the crucial ingrendients are considered in the data set.

After dropping alcohol, which one can easily check to have the most dominant influence on taste (given that the relation between the full variable $\bX$ and $Y$ has been unconfounded), we obtained 
$\hat{\beta} =0.62$, which sounds sensible since the set of predictor variables is no longer sufficient. 
When we dropped one of then other $X_j$, we always obtained $\hat{\beta}$ zero or close to zero (in one case). Since the other variables influence the taste
much weaker than $X_{11}$, the algorithm is not able to detect any significant confounding.

\subsection{Data sets with shuffling the target variable}

Here we describe a family of experiments where 
each single one cannot be assessed but one can discuss whether
the collection of results seem sensible.

If a
data set contains $d+1$ correlated variables $X_1,\dots,X_{d+1}$ we can 
take each $X_j$ as hypothetical target variable $Y$
and the remaining variables $\bX^{(j)}:=(X_1,\dots,X_{j-1},X_{j+1},\dots,X_{d+1})$ as hypothetical causes. 
Although we do not know whether some of these $d+1$ choices are purely causal in the sense that $X^{(j)}$
influences $Y^{(j)}:=X_j$ without confounder, we know that not all of them are purely causal because not all the variables can be a sink node of the underlying causal DAG.

Since our model uses independent sources
as in Independent Component Analysis (ICA)
as basis it is natural to apply our method to 
data sets that have been used in the context of ICA, for instance data from Magnetoencephalographic Recordings (MEG)\footnote{The data set is available at \url{http://research.ics.aalto.fi/ica/eegmeg/MEG_data.html}}
used by \citet{icaMeg}. The data set contains a data matrix with
$17,730$ samples of recordings from $122$ channels
in a
whole-scalp Neuromag-122 neuromagnetometer. 
We have used the first $10$ channels as $X_1,\dots,X_{d+1}$  and took each of it as potential target
and the remaining ones as potential causes.
We then obtained for $j=1,\dots,10$ the results 
$\hat{\beta} = 1.0$, $1.0$, $1.0$, $1.0$, $0.0$, $0.1$, $0.6$, $1.0$, $1.0$, $1.0$. 
We do not know the ground truth, but it sounds
reasonable that most of the cases are considered strongly confounded by the algorithm. 

\section{Discussion}

We have shown that our idealized model assumptions  
make it possible to infer whether the observed  correlations between the multi-dimensional predictor and the target variable are 
truly causal or an artifact of confounding or overfitting. 
For our assumptions, both cases of artifacts yield a  `dependence'
between the covariance matrix of the potential cause and the regression vector for predicting the effect from the potential cause. Here, `dependence' has the very simple meaning that principal components corresponding to small eigenvalues being over-represented in the decomposition of the regression vector while the meaning of `dependence' for the scenario from \citet{multivariateConfound} is more complex.

In our real data experiments, confounding seemed to be often underestimated, which suggests that real data generating process
deviate from the model assumptions in a way that the effect
of confounding is less visible by our method than the model
predicts. 
Despite these limitations,  
our findings may inspire further search for 
hidden causal information in high-dimensional data and provide an intuition about the relevance of concentration of measure effects in causal inference.


\section{Appendix}

\subsection{Proof of Lemma~1}

We first write $\Phi$ as
$
\Phi(v) = g(Av) Av,
$
with $g(w):=1/\|w\|$.   
Let $t \mapsto s(t)$ be some curve on the unit sphere $S^{d-1}$
and $\tilde{s}(t):=\Phi(s(t))$ its image. 
Then we have
\begin{eqnarray*}
\frac{d}{dt} \Phi(s(t)) &=& \langle \nabla g(As(t)), As'(t)\rangle As(t)\\
&&+ g(As(t)) As'(t),
\end{eqnarray*}
with
$
\nabla g(w) = -w/\|w\|^3 
$.
Hence we obtain
\begin{eqnarray}
&&\frac{d}{dt} \Phi(s(t)) \\
&=& \nonumber
\frac{-1}{\|As(t)\|^3 } \langle  As(t)   , As'(t)\rangle As(t)\\ \nonumber
&&+ g(As(t)) As'(t)  \\ \nonumber
&=& g(As(t)) \left(A s'(t) - \tilde{s}(t) \tilde{v}(t)^T s'(t)\right)\\ 
&=&   \label{eq:projformula}
g(As(t)) \left({\bf 1} - \tilde{s}(t) \tilde{s}(t)^T \right) As'(t).
\end{eqnarray}
Note that the matrix ${\bf 1} - \tilde{s}(t) \tilde{s}(t)^T$ projects $As'(t)$ onto the tangent space at $\tilde{s}(t)$ and
\eqref{eq:projformula} describes the Jacobian $D\Phi$
which maps between tangent spaces $T_{s(t)}$ and $T_{\tilde{s}(t)}$ at $s(t)$ and $\tilde{s}(t)$, respectively.  
Let $e_1,\dots,e_{d-1}$ and $\tilde{e}_1,\dots,\tilde{e}_{d-1}$ be orthonormal bases of $T_v$ and $T_{\tilde{v}}$, respectively. If we set $U_v:=(e_1,\dots,e_{d-1})$ and $U_{\tilde{v}}:=(\tilde{e}_1,\dots,\tilde{e}_{d-1})$, the Jacobian with respect to these bases 
reads
\[
\widehat{D\Phi}(v) := g(Av) U^T_{\tilde{v}} A U_v.
\]
We then have
\[
\det \widehat{D\Phi} (v) = g(Av)^{d-1} \det (U^T_{\tilde{v}} A U_v ). 
\]
Multiplying the equation $\tilde{v} =Av/\|Av\|$ with $A^{-1}$ and taking the norm on both sides yields 
\begin{equation}\label{eq:Ainvnorm}
1/\|Av\| = \|A^{-1} \tilde{v}\|.
\end{equation}
We thus obtain
\begin{eqnarray}\nonumber
p(\tilde{v}) &=& |\det \widehat{D\Phi}(\Phi^{-1}(\tilde{v}))|^{-1} \\
&=& \nonumber
\left(\|A^{-1}\tilde{v} \|^{d-1} |\det (U^T_{\tilde{v}} A U_v)|\right)^{-1}\\
&=& 
\left(\|A^{-1}\tilde{v} \|^{d-1} |\det (\tilde{A})|\right)^{-1}, \label{eq:density2}
\end{eqnarray}
with the abbreviation $A':=U^T_{\tilde{v}} A U_v$.
Let 
us now define the orthogonal $d\times d$ matrices
\[
W_v:=(U_v,v) \quad \hbox{ and } \quad (U_{\tilde{v}},\tilde{v}).
\]
Then we define
$
A' := W^T_{\tilde{v}} A W_v,
$
which implies 
$
|\det (A')| =|\det (A)|.
$
$A'$ can be written as
\[
A' = \left(\begin{array}{cc} \tilde{A} & 0 \\ w & \|Av\| \end{array}\right),
\]
where $w$ is some $1\times (d-1)$-matrix and 
$\tilde{A}:=U_{\tilde{v}}^T A U_v$. 
Hence we obtain
\[
\det (A') = \det (\tilde{A}) \|Av\| = \frac{\det (\tilde{A})}{\|A^{-1} \tilde{v}\|}, 
\]
where we have used also \eqref{eq:Ainvnorm}.
We can thus rewrite \eqref{eq:density2}  as
\[
p(\tilde{v}) = \frac{1}{|\det(A)|\|A^{-1}\tilde{v}\|^d} 
\] 

\subsection{Proof of Theorem~3}

By definition, $p_{\theta'}$ is obtained by applying the map
$\sqrt{R_\theta'}$ to vectors drawn from a rotation invariant distribution with renormalizing it later. Without loss of generality, let $R_\theta$ be diagonal with eigenvalues
$f_j(\theta)$. 
Let
$v$ be generated by drawing each entry $v_j$ from $\cN(0,1)$. 
We can then compute  the entries of $\tilde{v}$ by 
\[
\tilde{v}_j := \frac{1}{\sum_{i=1}^d  f_j(\theta') }\sqrt{f_j(\theta')} v_j.
\]
Rewriting (10) in terms of $v_j$ instead of $\tilde{v}$
yields
\begin{eqnarray*}
&&\log p_\theta (\tilde{v}) = -\frac{1}{2}\left\{ \log \frac{1}{d} \sum_{j=1}^d f_j(\theta') f_j(\theta)^{-1} v_j^2 \right. \\
&&-
\left. \log 
\frac{1}{d} \sum_{j=1}^d f_j(\theta') v_j^2 \right\} + \log \det R_\theta.
\end{eqnarray*}
Due to Chebychev's inequality we have 
\[
\left| \frac{1}{d} \sum_{j=1}^d f_j(\theta') f_j(\theta)^{-1} v_j^2
- \tau (R_{\theta'} R^{-1}_{\theta}) \right| \leq \delta,
\]
with probability $1-  \frac{1}{d^2} \sum_{j=1}^d f_j(\theta')^2
f_j(\theta)^{-2}/\delta^2= 1- \frac{1}{d} \tau(R^2_{\theta'} R_{\theta}^{-2})/\delta^2 $.  Likewise, 
\[
\left| \frac{1}{d} \sum_{j
=1}^d f_j(\theta') v_j^2
- {\rm tr} (R_{\theta'}) \right| \leq \delta,
\]
with probability $1- \frac{1}{d} (\tau(R_{\theta'}^2)/\delta^2$.  
To ensure that 
\begin{equation}\label{eq:ineqChev1}
\left| \log \frac{1}{d} \sum_{j=1}^d f_j(\theta') f_j(\theta)^{-1} v_j^2 - \log \tau (R_{\theta'} R_{\theta}^{-1})\right|\leq \epsilon,
\end{equation}
we need to ensure that $\delta \leq \epsilon/ \tau (R_{\theta'} R_{\theta}^{-1}) -\delta$ which can be achieved by
$\delta \leq \epsilon \tau (R_{\theta'} R_{\theta}^{-1})/2$ for sufficiently small $\epsilon$. 
  Likewise, we can achieve that 
\begin{equation}\label{eq:ineqChev2}
\left| \log \frac{1}{d} \sum_{j=1}^d f_j(\theta') v_j^2 - \log \tau (R_{\theta'})\right|\leq \epsilon,
\end{equation}
if $\delta \leq \epsilon \tau (R_{\theta'})/2$. 
Thus, both inequalities \eqref{eq:ineqChev1} and \eqref{eq:ineqChev2} together  hold with probability 
at least 
\[
1 - \frac{4}{d \epsilon^2} \left( \frac{\tau (R_{\theta'}^2 R_{\theta}^{-2})}{\tau (R_{\theta'} R_{\theta})^2} 
+  \frac{\tau (R_{\theta'}^{2})}{\tau (R_{\theta'})^2} \right). 
\]


\end{document}